\documentclass{IEEEtran}
\usepackage[utf8]{inputenc}
\usepackage{amsmath,bm,amsthm,amsfonts,amssymb}
\usepackage{graphicx}
\usepackage{xcolor}
\usepackage{xspace}
\usepackage{hyperref}
\usepackage{subcaption}
\usepackage{mathtools}

\newcommand{\ie}{i.e.,\xspace}
\newcommand{\eg}{e.g.,\xspace}

\renewcommand{\paragraph}[1]{\vspace{3mm}\noindent\textbf{#1}}

\renewcommand{\b}[1]{{\bm{#1}}}   
\renewcommand{\(}{\left(}           
\renewcommand{\)}{\right)}
\renewcommand{\[}{\left[}           
\renewcommand{\]}{\right]}           

\newcommand{\1}{\b{1}}              
\newcommand{\0}{\b{0}}              
\newcommand{\x}{\b{x}}

\newcommand{\w}{\b{w}}
\newcommand{\y}{\b{y}}
\renewcommand{\u}{\b{u}}

\renewcommand{\L}{\b{L}}            
\newcommand{\U}{\b{U}}              
\newcommand{\bSigma}{\b{\Sigma}}    
\newcommand{\bLambda}{\b{\Lambda}}  
\newcommand{\bOmega}{\b{\Omega}}  

\renewcommand{\H}{{\b{H}}}
\newcommand{\I}{\b{I}}
\newcommand{\X}{\b{X}}
\newcommand{\Y}{\b{Y}}
\newcommand{\W}{\b{W}}              

\DeclareMathOperator*{\argmin}{arg\,min}

\newcommand*{\Scale}[2][4]{\scalebox{#1}{#2}}%

\newcommand{\LG}{\L_{\hspace{-1px}G}}              
\newcommand{\LJ}{\L_{\hspace{-1px}J}}              
\newcommand{\LT}{\L_{\hspace{-1px}T}}              

\newcommand{\TG}[2]{{\mathcal{T}_{#2}^{\hspace{1px}G}} #1 \hspace{0.2mm}}              
\newcommand{\TJ}[2]{{\mathcal{T}_{#2}^{\hspace{1px}J}} \hspace{0.5mm} #1 \hspace{0.2mm}}              
\newcommand{\TT}[2]{{\mathcal{T}_{#2}^{\hspace{1px}T}} #1 \hspace{0.2mm}}              

\newcommand{\UG}{\U_{\hspace{-1px}G}}              
\newcommand{\UJ}{\U_{\hspace{-1px}J}}              
\newcommand{\UT}{\U_{\hspace{-1px}T}}              

\newcommand{\GFT}[1]{\textrm{GFT}\hspace{-.0mm}\{#1\}}
\newcommand{\IGFT}[1]{\textrm{GFT}^{\hspace{0.2mm}\Scale[0.7]{-1}}\hspace{-.0mm}\{#1\}}

\newcommand{\JFT}[1]{\textrm{JFT}\hspace{-.0mm}\{#1\}}
\newcommand{\IJFT}[1]{\textrm{JFT}^{\hspace{0.2mm}\Scale[0.7]{-1}}\hspace{-0.0mm}\{#1\}}
\newcommand{\STFT}[1]{\textrm{STFT}\hspace{-.0mm}\{#1\}}

\newcommand{\Rbb}{\mathbb{R}}

\newcommand{\E}[1]{\mathbb{E}\left[#1\right]}        
\renewcommand{\mod}[2]{\textrm{mod}\hspace{-.5mm}\(#1,#2\)}           
\renewcommand{\vec}[1]{\textrm{vec}\hspace{-.5mm}\(#1\)}           
\newcommand{\mat}[1]{\textrm{mat}\hspace{-.5mm}\(#1\)}           
\newcommand{\diag}[1]{\textrm{diag}\hspace{-.5mm}\(#1\)}           
\newcommand{\trace}[1]{\textrm{tr}\hspace{-.5mm}\(#1\)}            
\newcommand{\transpose}{\intercal}                      
\newcommand{\hermitian}{*}                      
\newcommand{\delequal}{\overset{\Delta}{=}} 

\newtheorem{theorem}{Theorem} \newtheorem{definition}{Definition}

\newtheorem{example}{Example}

\title{Towards stationary time-vertex signal processing}
\author{ Nathana{\"e}l Perraudin$^*$ and Andreas Loukas$^*$\thanks{A. Loukas and N. Perraudin contributed equally to this work.} and Francesco Grassi and Pierre Vandergheynst \\
\'{E}cole Polytechnique F\'{e}d\'{e}rale Lausanne, Switzerland}
\date{May 2016}

\begin{document}
\maketitle

\begin{abstract}
Graph-based methods for signal processing have shown promise for the analysis of data exhibiting irregular structure, such as those found in social, transportation, and sensor networks. Yet, though these systems are often dynamic, state-of-the-art methods for signal processing on graphs ignore the dimension of time, treating successive graph signals independently or taking a global average.  
To address this shortcoming, this paper considers the statistical analysis of time-varying graph signals. We introduce a novel definition of \emph{joint (time-vertex) stationarity}, which generalizes the classical definition of time stationarity and the more recent definition appropriate for graphs. Joint stationarity gives rise to a scalable Wiener optimization framework for joint denoising, semi-supervised learning, or more generally inversing a linear operator, that is provably optimal.
Experimental results on real weather data demonstrate that taking into account graph and time dimensions jointly can yield significant accuracy improvements in the reconstruction effort.
\end{abstract}

\section{Introduction}

Whether examining opinion dichotomy in social networks~\cite{adamic2005political}, how traffic evolves in the roads of a city~\cite{mohan2008nericell}, or neuronal activation patterns present in the brain~\cite{huang2015graph}, much of the high-dimensional data one encounters exhibit complex non-euclidean properties. 
This realization has been the driving force behind recent efforts to re-invent the mathematical models used for data analysis. 
Within the field of signal processing, one of the main research thrusts has been to extend harmonic analysis to graph signals, \ie signals supported on the vertices of irregular graphs. 
The key breakthrough in the field has been the introduction of a notion of frequency appropriate for graph signals and of the associated graph Fourier transform (GFT). Because it enables us to process signals taking into account complex relations between variables, the GFT has lead to advances in problems such as denoising~\cite{zhang2008graph} and semi-supervised learning~\cite{smola2003kernels,belkin2004semi}.     

Yet, state-of-the-art graph frequency based methods often fail to produce useful results when applied to real datasets. One of the main reasons underlying this shortcoming is that they ignore the time dimension, for example by treating successive signals independently or performing a global average~\cite{huang2015graph,perraudin2016stationary,kalofolias2016learn}. On the contrary, many of the systems to which graph signal processing is applied to are dynamic. 
Consider for instance a sensor network, and suppose that we want to infer the weather conditions on a mountain given temperature measurements from a small set of weather stations. Approaches that do not take into account the temporal evolution of weather will be biased by seasonal variations and unable to provide insights about transient phenomena. Moreover, when the weather dynamics are slow and predictable, taking into account the time dimension, \eg by imposing a smoothness prior, can yield accuracy improvements in the reconstruction effort.

Motivated by this need, this paper considers the \emph{statistical analysis of time-evolving graph signals}. Our results are inspired by the recent introduction of a joint temporal and graph Fourier transform (JFT), a generalization of GFT appropriate for time-varying graph signals~\cite{loukas2016frequency}, and the recent generalization of stationarity for graphs~\cite{perraudin2016stationary,girault2015stationary,marques2016stationary}. Our main contribution is a novel definition of time-vertex (wide-sense) stationarity, or \emph{joint stationarity} for short. We believe that the proposed definition is natural, at it elegantly generalizes existing definitions of stationarity in the time and vertex domains. We show that joint stationarity carries along important properties classically associated with stationarity. Moreover, our definition leads to a Wiener framework for solving denoising and interpolating time-varying graph signals, that yields superior performance compared to state-of-the-art methods in time or vertex domains. The proposed framework is composed out of two key components: a scalable joint power spectral density estimation method, and an optimization framework suitable for deconvolution under additive error. The latter  is shown to be optimal in the mean-squared error sense. Experiments with a real weather dataset illustrate the superior performance of our method, ultimately demonstrating that joint stationarity is a useful assumption in practice.

\section{Preliminaries}

Our objective is to model and predict the evolution of graph signals, \ie signals supported on the vertices $\mathcal{V} = \{ v_1, v_2, \ldots, v_N \}$ of a weighted undirected graph $G = (\mathcal{V}, \mathcal{E}, \W_G)$, with $\mathcal{E}$ the set of edges and $\W_G$ the weighted adjacency matrix.
A more convenient matrix representation of $\mathcal{G}$ is the (combinatorial\footnote{Though we use the combinatorial Laplacian in our presentation, our results are applicable to any positive semi-definite matrix representation of a graph or to the recently introduced shift operator~\cite{sandryhaila2013discrete}.}) Laplacian matrix $\LG = \diag{\W_G \1_N} - \W_G$, where $\1_N$ is the all-ones vector of size $N$, and $\diag{\W_G \1_N}$ is the diagonal degree matrix. 

\paragraph{Harmonic vertex analysis.} In the context of graph signal processing, the importance of the Laplacian matrix stems from it giving rise to a graph-specific notion of frequency. The \emph{Graph Fourier Transform} (GFT) of a graph signal $\x \in \mathbb{R}^N$ is defined as $\GFT{\x} = \UG^* \x$, where $\UG$ is the eigenvector matrix of $\LG$ and thus $\LG = \UG \bLambda_G \UG^*$.
The GFT allows us to extend filtering to graphs~\cite{hammond2011wavelets,shuman2013emerging,shuman2013vertex}. Filtering a signal $\x$ with a graph filter $h(\LG)$ corresponds to element-wise multiplication in the spectral domain
\begin{equation*}
    h(\LG)  \x \delequal \IGFT{h(\bLambda_G) \circ \GFT{\x}} = \UG h(\bLambda_G) \UG^\hermitian \, \x,
\end{equation*}
where the scalar function $h : \mathbb{R}_+ \mapsto \mathbb{R}$, referred to as the graph frequency response, has been applied to each diagonal entry of $\bLambda_G$. It is often convenient to represent the diagonal of matrix $h(\bLambda_G)$ as a vector, in which case we write $\b{h} = \diag{h(\bLambda_G)}$. The notation $\UG^\hermitian$ denotes the transposed complex conjugate of $\UG$, $\UG^\transpose$ the transpose of $\UG$, and $\bar{\UG}$ the complex conjugate of $\UG$.
We will also use the notion of \emph{graph localization}~\cite{shuman2013vertex,perraudin2016stationary}, a generalization of the translation operator used in the classical setting appropriate for graphs\footnote{Stationarity is classically defined as the invariance of statistical moments of a signal with respect to translation. This definition however cannot be directly generalized to graphs, which do not possess regular structure and thus lack of an isometric translation operator.}.  
Localizing a filter with frequency response $h$ onto vertex $v_i$ reads
\begin{equation}
\label{def:localization_operator}
    \TG{h}{i} \delequal h(\LG) \, \b{\delta}_i,
\end{equation}
where $\b{\delta}_i$ is a Kronecker delta centered at vertex $v_i$. 
For a sufficiently regular function $h$, this operation localizes the filter around $v_i$~\cite[Theorem 1 and Corollary 2]{shuman2013vertex}. 
Evaluated at the $i_2$-th vertex, the above expression becomes 
\begin{equation} \label{def:localization_operator_el}
\TG{h}{i_1} (i_2) = \sum_{n=1}^{N} h(\lambda_{n})\, \bar{\u}_{n}(i_1)\, \u_{n}(i_2),
\end{equation}
where we use the notation $\u_n(i) = [\UG]_{i,n}$ and $\lambda_n = {[\bLambda_G]}_{n,n}$. Note that in the expression above, the localization operator takes precedence over indexing and $\TG{h}{i_1} (i_2) = [\TG{h}{i_1}] (i_2)$; this convention is used throughout this paper.  
The concept of localization is intimately linked to that of \emph{translation} in the time domain. If $\TT{h}{\tau}$ is the localization operator taken on a cycle graph of $T$ vertices (representing time), localization is equivalent to translation 
\begin{equation}
\label{def:localization_operator_time}
    \TT{h}{\tau}(t) = \TT{h}{0}(t-\tau), \quad \text{for all} \quad t, \tau = 1, \ldots, T.
\end{equation}
In simple words, for a cyclic graph, the localization operator computes the inverse Fourier transform of the frequency response $h$, and translates it to vertex $v_i$. We can verify this using the fact that the complex exponential Fourier basis form the eigenvector set of all cyclic graphs~\cite{strang1999discrete}, which together with~\eqref{def:localization_operator_el} implies that
\begin{align}
    \TT{h}{t_1} (t_2) &= \frac{1}{T}\sum_{\tau=1}^{T} h(\omega_{\tau})e^{-2\pi j\frac{(\tau-1) t_1}{T}} \, e^{2\pi j\frac{(\tau-1) t_2 }{T}} \notag \\
    &\hspace{-5mm}= \frac{1}{T}\sum_{\tau=1}^{T} h(\omega_{\tau}) \, e^{2\pi j\frac{(\tau-1) (t_2-t_1)}{T}} = \TT{h}{0} (t_2-t_1).
\end{align}
Above $\TT{h}{0} = \UT \b{h}$ is the inverse Fourier transform of $\b{h}$ and $\UT$ is the orthonormal Fourier basis.
In the case of irregular graphs, localization differs further from translation because the shape of the localized filter adapts to the graph and varies as a function of its topology. Additional insights about the localization operator can be found in~\cite{shuman2013vertex,hammond2011wavelets,perraudin2016global,perraudin2016stationary}.

\paragraph{Harmonic time-vertex analysis.} Suppose that a graph signal $\x_t$ is sampled at $T$ successive regular intervals of unit length. The time-varying graph signal $\X = \left[ \x_1, \x_2, \ldots, \x_T \right] \in \mathbb{R}^{N\times T}$ is then the matrix having graph signal $\x_t$ as its $t$-th column. Equivalently, $\X = \left[ \x^1, \x^2, \ldots, \x^N\right]^\transpose$ holds $N$  temporal signals $\x^i \in \Rbb^T$, one for each vertex $v_i$. Throughout this paper, we denote as $\x = \vec{\X}$ (without subscript) the vectorized representation of the matrix $\X$. 

The frequency representation of $\X$ is given by the joint (time-vertex) Fourier transform (or JFT for short) 
\begin{align}
    \JFT{\X} = \UG^* \X \bar{\U}_T,
\end{align}
where, once more, $\UG$ is the graph Laplacian eigenvector matrix, whereas $\bar{\U}_T$ is the complex conjugate of the DFT matrix divided by $1/\sqrt{T}$. 
In fact, matrix $\UT$ is the eigenvector matrix of the lag operator (or Laplacian matrix) $\LT = \UT \, \bLambda_T \, \UT^*$. Denote by $\bOmega$ the diagonal matrix of angular frequencies (\ie $\bOmega_{tt} = \omega_t = 2 \pi t/T$). In case $\LT$ is the lag operator, we have  $\bLambda_T = e^{-j \bOmega}$, where $j = \sqrt{-1}$. When $\LT$ is the Laplacian, we have $\bLambda_T = \textrm{real}\(\b{I} - e^{-j \bOmega}\)$.
Expressed in vector form, the joint Fourier transform becomes $\JFT{\x} = \UJ^* \x$, where $\UJ = \UT \otimes \UG $ is unitary, and operator $(\otimes)$ denotes the kroneker product. The inverse joint Fourier transforms in matrix and vector form are, respectively, $\IJFT{\X} = \UG \X \UT^\transpose $ and $\IJFT{\x} = \UJ \x$. For an in-depth discussion of JFT and its properties, we refer the reader to~\cite{loukas2016frequency}.

Leveraging the definition of the JFT, filtering and localization can also be extended to the joint (time-vertex) domain. A \emph{joint filter} $h(\L_J)$ is a function defined in the joint spectral domain $h: \Rbb_+ \times \Rbb \mapsto \Rbb$ that is evaluated at the graph eigenvalues $\lambda_G$ and the angular frequencies $\omega$.  The output of a joint filter is 
\begin{equation}  \label{eq:def_joint_filtering}
    h(\LJ) \x \delequal \UJ \, h(\bLambda_G,\bOmega) \, \UJ^* \x, 
\end{equation} 
where $h(\bLambda_G,\bOmega)$ is a $NT\times NT$ diagonal matrix with  $[h(\bLambda_G,\bOmega)]_{k,k} = h(\lambda_n,\omega_\tau)$ and $k = N(\tau-1)+n$. Equivalently, if we define the matrix $\H$ of dimension $N \times T$ as $\H_{n,\tau} = h(\lambda_n,\omega_\tau)$ for every graph frequency $\lambda_n$ and temporal frequency $\omega_\tau$, we have 
\begin{equation}
         h(\LJ) \x  = \vec{ \IJFT{ \H \circ \JFT{\X})} },
\end{equation}
with $(\circ)$ being the element-wise multiplication (Hadamard product).
In an analogy to~\eqref{def:localization_operator}, we define the \emph{joint localization operator} as
\begin{eqnarray} 
    \TJ{h}{i,t}  & \delequal & \mat{ h(\LJ) \, (\b{\delta}_{t} \otimes \b{\delta}_{i})}  \label{eq:joint_localization} \\
    &= &  \, \IJFT{ \H \circ \JFT{\b{\delta}_{i} \b{\delta}_{t}^\transpose})}  
    \label{eq:joint_localization_mat}
\end{eqnarray}
where $\mat{\cdot}$ is the matricization operator, such that $\mat{\vec{\X}} = \X$.
In order to link \eqref{eq:joint_localization} with graph localization \eqref{def:localization_operator} and the classical translation operator, we observe the following relations 
\begin{align}
 \TJ{h}{i_1,t_1}(i_2,t_2) 
    &= \frac{1}{T} \hspace{0mm} \sum_{\substack{n=1\\\tau = 1}}^{N, T} \hspace{0mm} h(\lambda_{n},\omega_{k})\bar{\u}_{n}(i_{1})\u_{n}(i_{2})e^{2\pi j\frac{(\tau-1) (t_{2}-t_{1})}{T}} \nonumber\\
    &\hspace{-12mm}= \TJ{h}{i_1,0} (i_2, t_2-t_1) \label{eq:joint-translation-property} \\
    &\hspace{-12mm}= \frac{1}{T}\sum_{n=1}^{N} \[\sum_{\tau=1}^{T}h(\lambda_{n},\omega_{\tau})e^{2\pi j\frac{(\tau-1) (t_{2}-t_{1})}{T}}\] \bar{\u}_{n}(i_{1})\u_{n}(i_{2}) \nonumber \\
    &\hspace{-12mm}= \sum_{n=1}^{N} \[\TT{\H_{n,\cdot}}{t_1}\]\hspace{-0.5mm}(t_2) \, \bar{\u}_{n}(i_{1})\u_{n}(i_{2}) \label{eq:translation-first} \displaybreak \\
    &\hspace{-12mm}= \frac{1}{T}\sum_{\tau=1}^T \[\sum_{n=1}^N h(\lambda_{n},\omega_{\tau})\bar{\u}_{n}(i_1)\u_{n}(i_2)\] e^{2\pi j\frac{(\tau-1) (t_2-t_1)}{T}}  \nonumber \\
    &\hspace{-12mm}= \frac{1}{T}\sum_{\tau=1}^T \[\TG{\H_{\cdot,\tau}}{i_1}\]\hspace{-0.5mm}(i_2) \,  e^{2\pi j\frac{(\tau-1) (t_{2}-t_{1})}{T}}.  \label{eq:localization-first}
\end{align}
The above equations provide three key insights about the joint localization operator: 
\begin{enumerate}
\item From \eqref{eq:joint-translation-property}, we observe that the localization operator performs a translation along the time dimension. 
\item From \eqref{eq:translation-first}, it follows that joint localization consist of first localizing (translating) independently in time each line of the matrix $\H$ and then localizing independently on the graph each column of the resulting matrix. Joint localization is thus equivalent to a successive application of a graph and and a time localization operator. 
\item Furthermore, according to \eqref{eq:localization-first}, the successive localization in time and graph can be performed in any order.
\end{enumerate}
When the filter is separable, \ie when the joint frequency response can be written as the product of a frequency response defined solely in the vertex domain and one in the time domain $h(\lambda,\omega) = h_1(\lambda) h_2(\omega)$, the joint localization is simply
\begin{equation}
    h(\LJ) \, (\b{\delta}_{t} \otimes \b{\delta}_{i})  = \vec{ h(\LG)  (\b{\delta}_{t} \otimes \b{\delta}_{i}^\transpose) h(\LT)}.
\end{equation}
Nevertheless, for this work, we assume that the filter is not separable as it is a too restrictive hypothesis.

\section{Joint Time-Vertex Stationarity}
\label{sec:stationarity}

Let $\X$ be a discrete multivariate stochastic process with finite number of time-steps $T$ that is indexed by vertex $v_i$ and time $t$. We refer to such processes as \emph{joint time-vertex processes}, or \emph{joint} processes for short. 

To put our results in context, let us first review the established definitions of stationarity over time and vertex domains, respectively. Our definition will emerge us a consequence of both. We note that, although our exposition is self-contained, the reader will benefit from familiarizing with previous work on stationarity on graphs~\cite{perraudin2016stationary}. 

\begin{definition}[Time stationarity] \label{def:time-stationarity}
 A joint process $\X$ is Time Wide-Sense Stationary (TWSS), if and only if the following two properties hold independently for each vertex $v_i$: 
\begin{enumerate}
    \item The expected value is constant over the time domain $$\E{ \x^{i} } = c_i \1_T.$$ 
    \item There exists a function $\gamma_i$, for which $$ \left[ \bSigma_{\x^i} \right]_{t,\cdot}
    =  \left[ \mathbb{E} \big[\x^{i} {\x^i}^* \big] - \mathbb{E}\big[\x^{i}\big] \mathbb{E}\big[{\x^{i}}^*\big] \right]_{t,\cdot}= 
 \TT{\gamma_i}{t}$$
\end{enumerate}
Function ${\gamma}_i$ is the autocorrelation function of signal $\x^i$ in the Fourier domain, and is also referred to as Time Power Spectral Density (TPSD).
\end{definition}
We remind the reader that on a cyclic graph, localizing the TPSD is equivalent to translating the autocorrelation. Thus, using \eqref{eq:joint_localization_mat} we recover the classical definition, where the autocorrelation function depends only on the time difference: $\left[ \bSigma_{\x^i} \right]_{t,\tau} = \TT{\gamma_i}{0}(t-\tau)$. Simply put, assuming time stationarity is equivalent to asserting that the statistics of the two first moments are independent of the time. We also observe that the TPSD is the Fourier transform of the autocorrelation, agreeing with the Wiener-Khintchine Theorem~\cite{Wiener1930generalized}. 
To summarize, TPSD encodes the statistics of the signal in the spectral domain. 

This consideration allows us to generalize the concept of stationarity to graph signals. Please refer to the work of Perraudin and Vandergheynst~\cite{perraudin2016stationary} for a more detailled study. We express a variation of their definition in the following.
\begin{definition}[Vertex stationarity] \label{def:vertex-stationarity}
A joint process $\X = [\x_1 \x_2 \ldots \x_T]$ is called Vertex Wide-Sense (or second order) Stationary
(VWSS), if and only if the following two properties hold independently for each time $t$:
\begin{enumerate}
    \item The expected value is in the null space of the Laplacian $$\LG \E{\x_{t}} = \0_{N}.$$
    \item There exists a graph filter $s_t(\L_G)$, for which $$\left[\bSigma_{\x_t}\right]_{i,\cdot} = \left[ \E{ \x_{t} \x_{t}^* } - \E{ \x_{t}} \E{\x_{t}^*} \right]_{i,\cdot} = \TG{s_t}{i}.$$
\end{enumerate}
Function ${s}_t$ is the autocorrelation function of signal $\x_t$ in the graph Fourier domain and is also referred to as Vertex Power Spectral Density (VPSD).
\end{definition}
\vspace{2mm}

Considering that the null space of $\LT$ in both the normalized and the combinatorial case is the span of the constant eigenvector $\1_T$, the first condition of the above definition is analogous to the corresponding condition of the time stationarity definition. Moreover, the condition for the second moment is a natural generalization of the second condition of time stationarity where, instead of imposing translation invariance, we suppose invariance under the localization operator. This second condition is in fact equivalent to a generalization of the Wiener-Khintchine theorem and implies that $\bSigma_{\x_t}$ is jointly diagonalizable with $\LG$ (in TWSS the covariance is Toeplitz and thus also diagonalizable with the DFT matrix $\U_T$). 

We now unify the TWSS and GWSS in order to leverage both the time and vertex domain statistics.
\begin{definition}[Joint stationarity] \label{def:time-vertex-stationarity}
A process $\X$ is called Jointly (or time-vertex) Wide-Sense Stationary (JWSS), if and only if its vector form $\x = \vec{\X}$  satisfies the following properties:
\begin{enumerate}
     \item The expected value is in the null space of the Laplacian $$\LJ \E{\x } = \0_{NT}.$$
    \item There exists a joint filter $h(\L_J)$, for which $$\left[ \bSigma_{\x} \right]_{k,\cdot} = \left[ \E{\x \x^*} - \E{\x} \E{\x^*} \right]_{k,\cdot} = \vec{ \TJ{h}{i,t}},$$
    where $k = N (t-1) + i$.
\end{enumerate}
Function $h$ is the autocorrelation function of signal $\x$ in the joint Fourier domain and is also referred to as time-vertex power spectral density or Joint Power Spectral Density (JPSD) for short.

\end{definition}

The definition above is in fact equivalent to stating that the mean is constant, and the covariance matrix $\bSigma_{\x}$ is jointly diagonalizable with the joint Laplacian $\LJ$. The latter statement is (also) a generalization the Wiener-Khintchine theorem and is proven next. 

\begin{theorem}\label{theo:diagonalization}
A process $\X$ is JWSS if and only if 1) $\LJ \E{\x } = \0_{NT}$, and 2) its covariance matrix is jointly diagonalizable by the joint Fourier basis $\UJ$.
\end{theorem}
\begin{proof}
To prove an equivalence relation between the two definitions (\ie the JWSS definition and the one stated by the theorem) we will prove a one-to-one equivalence between their respective conditions. Clearly the first conditions of both definitions are identical. The second condition $\left[ \bSigma_{\x} \right]_{k,\cdot} = \vec{\TJ{h}{i,t}}$ of the joint stationarity definition together with~\eqref{eq:joint_localization} assert that the covariance being a joint filter $\bSigma_{\x} = h(\LJ)$ for some function $h$. We therefore have that $\bSigma_{\x} = \UJ h(\bLambda_G,\bOmega)  \UJ^*$ with $h(\bLambda_G,\bOmega)$ diagonal, which implies our claim.
\end{proof}

Interestingly, assuming joint stationarity is equivalent to assuming stationarity in both domains at the same time. 
\begin{theorem} \label{theo:time-vertex-def}
If a joint process $\X$ is JWSS, then it is both TWSS and GWSS.
\end{theorem}
\begin{proof}
It is straightforward to see that $\LJ \E{\x } = \0_{NT}$ if and only if both $\LT  \E{\x_{t}} = \0_{N}$ and $\LG  \E{\x^{i}} = \0_{T}$, hold for all $t$ and $i$. 
We still need to show that the second-order moment properties of TWSS and VWSS are equivalent to that of JWSS. If a process is joint stationary, then from \eqref{eq:localization-first} we have that, for each vertex $v_i$
\begin{align*}
     \left[\bSigma_{\x^i}\right]_{t_1,t_2} &= [\TJ{h](i,t_2)}{i,t_1} \\
     &= \frac{1}{T}\sum_{\tau=1}^T  \[\TG{ \H_{\cdot,\tau} }{i}\](i) \,  e^{2\pi j\frac{{(\tau-1)} (t_{2}-t_{1})}{T}}
\end{align*}
which is equivalent to asserting that $\x^i$ is stationary in time with TPSD $\gamma_i(\omega_\tau) = [\TG{ \H_{\cdot,\tau} }{i}] (i)$. Similarly, using~\eqref{eq:translation-first}, we find that for each time $t$
\begin{align*}
     \left[\bSigma_{\x_t}\right]_{i_1,i_2} = [\TJ{h](i_2,t)}{i_1,t} = \sum_{n=1}^{N} \[\TT{\H_{n,\cdot}}{t}\]\hspace{-0.5mm}(t) \, \bar{\u}_{n}(i_{1})\u_{n}(i_{2})
\end{align*}
meaning that process $\x_t$ is stationary with VPSD $s_t(\lambda_n) = [\TG{\H_{n,\cdot}}{t}](t)$.
\end{proof}

\begin{example}[White \emph{i.i.d.} noise]
White i.i.d. noise $\w \in \Rbb^{NT}$ is JWSS for any graph. Indeed, the first moment $\E{\w}$ is constant for any time and vertex. Moreover, due to being an identity matrix, the covariance of $\w$ is diagonalized by the joint Fourier basis of any graph $\bSigma_\w = \I = \UJ \I \UJ^*$. This last equation tells us that the JPSD is constant, which implies that similar to the classical case, white noise contains all joint (time-vertex) frequencies.   
\end{example}

An interesting property of JWSS processes is that stationarity is preserved through a filtering operation. 
\begin{theorem} \label{theo:time-vertex-psd-trans}
When a joint filter $f(\L_J)$ is applied to a JWSS process $\X$, the result $\b{Y}$ remains JWSS with mean $f(0,0)\E{\X}$ and JPSD that satisfies
\begin{equation}
    h_{\b{Y}}(\lambda,\omega) = f^2(\lambda,\omega) \cdot h_{\X}(\lambda,\omega).
\end{equation}
\end{theorem}
\begin{proof}
The output of a filter $f(\L_J)$ can be written in vector form as $\y = f(\LJ) $. If the input signal $\x$ is JWSS, we can confirm that the first moment of the filter output is zero, $\E{f(\LJ) \x }= f(\LJ) \E{\x }=f(0,0) \E{\x }$. The last equality follows from the fact that by definition $\E{\x }$ is in the null space of $\LJ$. The computation of the second moment gives
\begin{align*}
\bSigma_{\y} &= \E{ f(\LJ)\b{x} \( f(\LJ) \b{x} \)^* }  - \E{h(\LJ) \x } \E{ (f(\LJ) \x)^* }\\
    &= f(\LJ) \E{ \b{x} \b{x}^* } f(\LJ) - f(\LJ) \E{\x}\E{\x^*} f(\LJ)^*\\ 
    &=  f(\LJ) \b{\Sigma}_{\x} f(\LJ)^* \\
    & = \UJ \, \(f^2(\bLambda_G,\bOmega)\, h_\X(\bLambda_G,\bOmega) \) \, \UJ^*,
\end{align*}
which is, from Theorem~\ref{theo:diagonalization}, JWSS as it is diagonalizable by $\U_J$. 
\end{proof}

As the following diagram illustrates, Theorem \ref{theo:time-vertex-psd-trans} provides a simple way to artificially produce JWSS signals with a prescribed PSD $f^2$ by simply filtering white noise with the joint filter $f(\L_J)$.
\begin{center}
    \includegraphics[width=0.2\textwidth]{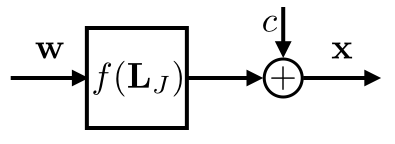}
\end{center}
The resulting signal will be stationary with PSD $f^2$ and this holds for white noise abiding to any distribution (not only Gaussian). In the sequel, we assume for simplicity that the signal is centered at $0$, \ie $\E{\x} = 0 \cdot \1$. 

Whenever it is clear from the context, in the following we simply refer to the TPSD, VPSD, and JPSD as PSD.

\section{Joint Power Spectral Density estimation}

As the JPSD is central in our method, we need a reliable way to compute it. Since we take into account the correlation both in the time and in the vertex domain, the actual size of the covariance matrix $\bSigma_{\x}$ is $N T\times NT$. In many cases, this matrix is not computable nor can be even stored. Additionally, if attempt to estimate it using classical covariance estimation methods, the number of samples necessary for obtaining a reasonable estimation accuracy can be prohibitive. The number of samples needed for obtaining a good sample covariance matrix of an $n$-dimensional process is generally not known, but for distributions with finite second moment it has been shown to be $O(n \log{n})$ by Rudelson~\cite{rudelson1998,vershynin2012}. In our case, this theorem implies that we need $O(NT \log{(NT)})$ signals, of $NT$ variables each, to obtain a good estimate of the statistics of a joint process.

To circumvent this issue, we leverage the time-vertex structure of the data. The basic idea behind our approach stems from two established methods used to estimate the TPSD of a temporal signal, namely Bartlett's and Welch's methods~\cite{bartlett1950periodogram}, which are summarized below. 

\paragraph{TPSD estimation methods.} In Bartlett's method, the signal (timeseries) is first cut into equally sized segments without overlap. Then, the Fourier transform of each segment is computed. Finally, the PSD  is obtained by averaging over segments the squared amplitude of the Fourier coefficients. Welch's method~\cite{welch1967use} is a generalization that works with overlapping segments. We can see the TPSD estimation of both methods as the averaging over time of the squared coefficients of a Short Time Fourier Transform (STFT). We remind the reader that STFT is used to extract the frequency content of a temporal signal at a given time, by first selecting a part of the signal using a window and then compute the discrete Fourier transform. More concretely, for a discrete signal $\b{s}$ of length $T$, the circular discrete sampled STFT of $\b{s}$ at the $m$-th (out of $M$) frequency band, and under window $\b{g}$ is
\begin{eqnarray*}
   \STFT{\b{s}}(k,m) \delequal \sum_{t=1}^{T}\b{s}(t) \, \overline{\b{g}\(t_k\)} \, e^{-2\pi j\frac{(t-1) (m-1)}{M}},
\end{eqnarray*} 
where $t_k = \mod{t-a(k-1)}{T}+1$, scalar $a$ is the shift in time between two successive windows~\cite[equation 1]{ltfatnote030}, and $\mod{t}{T}$ finds the remainder after division by $T$ \ie $\mod{t}{T}=t-T\lfloor \frac{t}{T} \rfloor$.
Note that $k=0,1,\dots,\lfloor \frac{T}{a} \rfloor -1$ is the time band centered at $k a$ and that $m=1,\dots,M$ is the frequency band index. For additional insights about this transform, we refer the reader to \cite{grochenig2013foundations,feichtinger2012gabor}.

\paragraph{Joint PSD estimation.} Based on the idea that the Bartlett method is an average of STFT coefficients, we propose to use the GFT of the STFT as a tool to estimate the joint PSD.
Consider a time window $\b{g}$ and a time-vertex signal $\X$. We first define the coefficients' tensor as 
\begin{align*}
{\b{C}}_{n,k,m} &\delequal \sum_{i=1}^N [\UG]_{i,n} \, \STFT{ \x^i }(k, m)  \\
    &= \sum_{i=1}^N [\UG]_{i,n} \sum_{t=1}^{T} \X_{i,t} \, \overline{\b{g}(t_k)} \, e^{-2\pi j \frac{(t-1) (m-1)}{M}}.
\end{align*}
A usual parameter for $M$ is the support size of $\b{g}$. Then, for half-overlapping windows, we set $a$ to $M/2$. For any discrete vertex frequency $\lambda_n$ and time frequency $\omega_m = 2 \pi m / M$, our JPSD estimator is
\begin{equation}
    \tilde{h}\(\lambda_n, \omega_m \) \delequal \frac{a}{T\|\b{g}\|_2^2} \sum_{k=0}^{\lfloor T/a \rfloor -1} {\b{C}}_{n,k,m}^2
\end{equation}
In order to get an estimate of $h$ at $\omega \neq \omega_m$, we interpolate between the known points. Alternatively, with sufficient computation power, one may set $M=T$. Though alternative choices are possible, we suggest using the iterated sine window 
$$
g(t) = \sin\(0.5\pi\cos\(\pi t / M\)^2\) \, \chi_{[-M/2,M/2]}(t),
$$
where $\chi_{[-M/2,M/2]}(t)=1$ if $t\in [-M/2,M/2]$ and $0$ otherwise, as it turns the STFT into a tight operator for $M=2a$.
We defer an error analysis of the estimator for the longer version of this paper.


\paragraph{Other PSD estimation methods.} In case $T \ll N$, two problems arise with the aforementioned method. First, we cannot compute the graph Fourier basis $\UG$, and second the number of sample in time might not be sufficient to average over time. To circumvent this problem, one can do the average over the graph vertex using a technique similar to the PSD estimator of \cite{perraudin2016stationary}. This technique will be studied in future work.

\section{Optimization framework}

We can leverage our definition of stationarity to generalize the optimization framework of~\cite{perraudin2016stationary}, useful for denoising, interpolating, and more generally deconvoling stationary processes.
Concretely, suppose that our measurements $\y$ are generated by a linear model
\begin{equation} \label{eq:general_data_model}
    \y = \b{A}\x+\w,
\end{equation}
where, as in the rest of this document, $\x$ and $\y$ are the vectorized version of $\X,\Y$. Further, suppose that the JPSD of $\x$ is $h_\X$, whereas the noise $\w$ is zero mean has JPSD $h_\W$ and may follow any distribution. Matrix $\b{A}$ is a general linear operator, not assumed to be jointly diagonalizable with $\LJ$. 

\paragraph{Tikhonov-regularization.} Whet the signal $\x$ varies smoothly on the graph, i.e is low frequency based, the classical approach of finding $\x$ from $\y$, consists of solving the following optimization scheme, commonly referred to as Tikhonov-regularization
\begin{equation} \label{prob:graph_tik}
    \argmin_\b{x} \| \b{A}\b{x}-\b{y}\|_2^2 + \alpha \, \b{x}^* \LJ \b{x}
\end{equation}
Notice that the prior above is separable into two terms
\begin{align}
    \x^* \LJ \x = \trace{ \X^* \LG \X } + \trace{ \X \LT \X^* }.
\end{align}
As a result, optimization problem~\eqref{prob:graph_tik} can only encode a particular joint time-vertex structure. Additionally this scheme requires the parameter $\alpha$ to be tuned and does not take into account the statistical structure of the signals. 

\paragraph{Wiener optimization framework.} We instead propose to recover $\x$ as the solution of the Wiener optimization problem
\begin{equation} \label{prob:Wiener-opt}
    \dot{\x} = \argmin_{\b{x}}  \|\b{A}\b{x} - \b{y}\|_2^2 + \|f(\LJ) (\b{x} -\E{\b{x}}) \|_2^2,
\end{equation}
where $f(\lambda,\omega)$ are the joint Fourier penalization weights, defined as 
\begin{align}
    f(\lambda,\omega) \delequal \left|\sqrt{\frac{h_\W(\lambda,\omega)}{h_\X(\lambda,\omega)}}\right|=\frac{1}{\sqrt{\mathrm{SNR}(\lambda, \omega)}}.
\end{align}
In the noise-less case, one alternatively solves the problem
\begin{equation} \label{prob:Wiener-opt-noiseless}
    \dot{\bf{x}} = \argmin_{\b{x}}  \|h_{\X}^{-\frac{1}{2}}(\LJ) \, \b{x} \|_2^2, \quad \text{subject to} \quad \b{A} \b{x} = \b{y}.
\end{equation}
Intuitively, the weight $f(\lambda, \omega)$ heavily penalizes frequencies associated with low SNR and vice-versa. Formally, we can show that: 
\begin{itemize}
	\item If $\X$ is a Gaussian process, then the solution of Problem~\eqref{prob:Wiener-opt} coincides with a MAP estimator.
	\item If $\b{A}$ is a masking operator, then the solution of Problem~\eqref{prob:Wiener-opt} coincides with the minimum mean square error linear estimator.
	\item If $\b{A} = a(\LJ)$ is a joint filter, then the solution of Problem~\eqref{prob:Wiener-opt} is a joint Wiener filter~\cite{loukas2016frequency}.
\end{itemize}
The proofs are generalizations of Theorems 3,4 and 5 of \cite{perraudin2016stationary}.

\paragraph{Comparison to the MAP estimator.} There are three main advantages of the Wiener optimization framework over a Gaussian MAP estimator based on an empirical covariance matrix estimate.
Firstly, assuming stationarity allows for a more robust estimate of the covariance matrix. This is crucial in this problem since we typically expect the number of variable $N\times T$ to be large and an empirical estimate of the covariance matrix to be expensive. 
Secondly, storing the covariance might not be possible as it consists of $O((NT)^2)$ elements. On the contrary, the JPSD $h_\X$ has only $NT$ elements.
Finally, thanks to proximal splitting methods, we can derive an algorithm for solving Problem~\eqref{prob:Wiener-opt} that requires only the application of $\b{A}$ and spectral graph filtering. On the contrary the classical Gaussian MAP estimator requires the inverse of a large part of the covariance matrix.

\section{Experiments}
We apply our methods to a weather dataset depicting the temperature of 32 weather stations, over a span of $31$ days. Our experiment aims to show that 1) joint stationarity is a useful model, even in datasets which may violate the strict conditions of our definition, and 2) that time-vertex stationarity can yield a significant increase in denoising and recovery accuracy, as compared to time- or vertex-based methods, on a real dataset. 
 
\paragraph{Experimental setup.} The French national meteorological service has published in open access a dataset\footnote{Access to the raw data is possible directly from \url{https://donneespubliques.meteofrance.fr/donnees_libres/Hackathon/RADOMEH.tar.gz}} with hourly weather observations collected during the Month of January 2014 in the region of Brest (France). The graph is built from the coordinates of the weather stations by connecting all the neighbors in a given radius with a weight function $[\W_G]_{i_1,i_2} = \mathrm{exp}({-k \, d(i_1,i_2)^2 })$, where $d(i_1,i_2)$ is the euclidean distance between the stations $i_1$ and $i_2$. Parameter $k$ is adjusted to as obtain an average degree around $3$ ($k$, however, is not a sensitive parameter). As sole pre-processing, we remove the mean (over time and stations) of the temperature. This is equivalent to removing the first moment. 

The dataset, which consisted of a total of $T = 744$ timesteps, was split into two parts of size $\rho T$ and $ (1 - \rho) T$, respectively.  We use the first part of the dataset to estimate the PSD and the second to quantify the joint filter performance. We compare our joint method to the state-of-the-art wiener filters for the disjoint time/vertex domains, which are known to outperform non-statistics based methods, such as graph/time Tikhonov and graph/time TV. To highlight the benefit of the joint approach, in the disjoint cases we use the entire dataset to estimate the PSD (for $\rho = 1$ the same data are used for both training and testing).   

\begin{figure}[t!]
\centering
\begin{subfigure}{0.9\columnwidth}
\hspace{-3mm}
\includegraphics[width=\columnwidth]{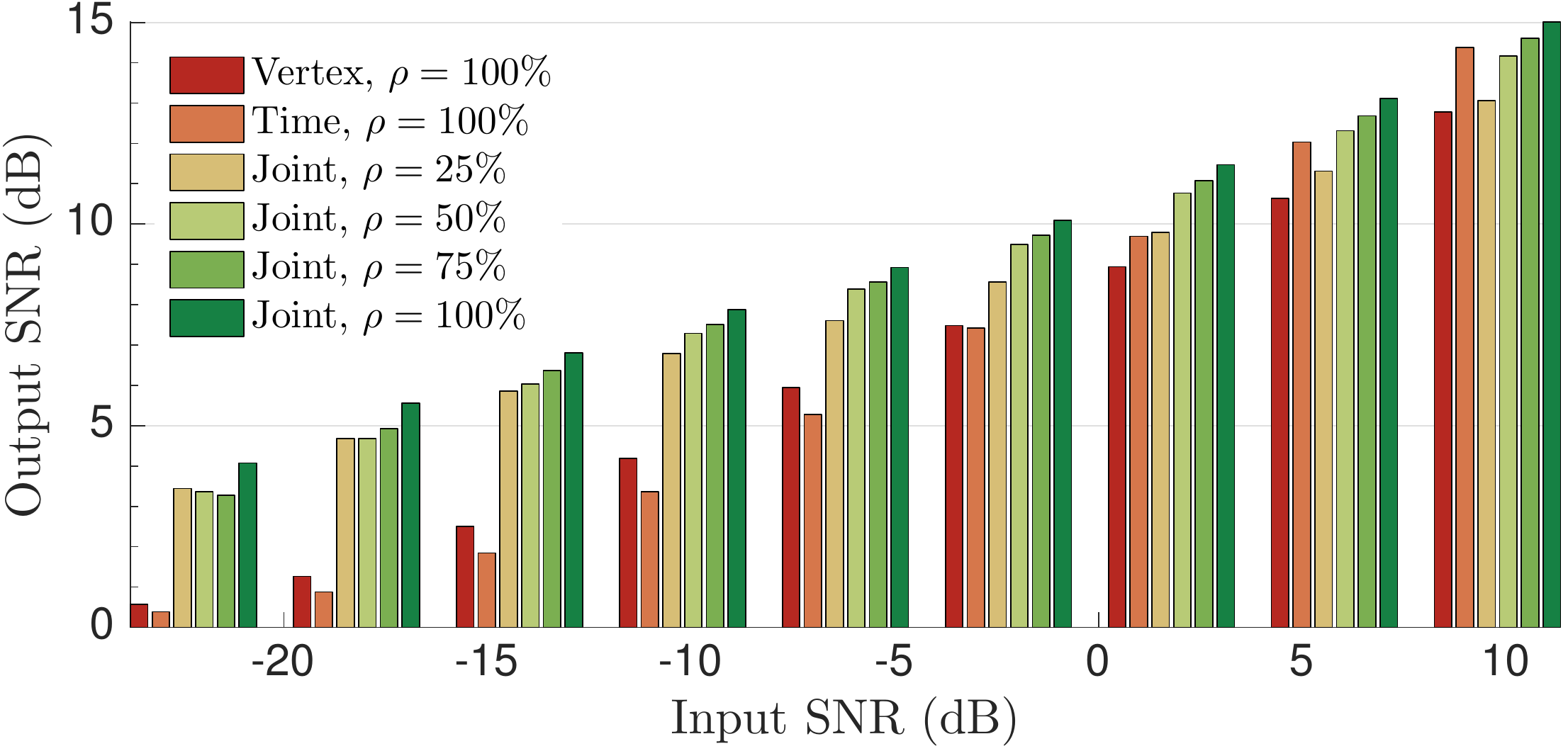}%
\caption{Denoising}%
\label{subfiga}%
\end{subfigure}\hfill%
\\
\begin{subfigure}{0.9\columnwidth}
\hspace{-3mm}
\includegraphics[width=\columnwidth]{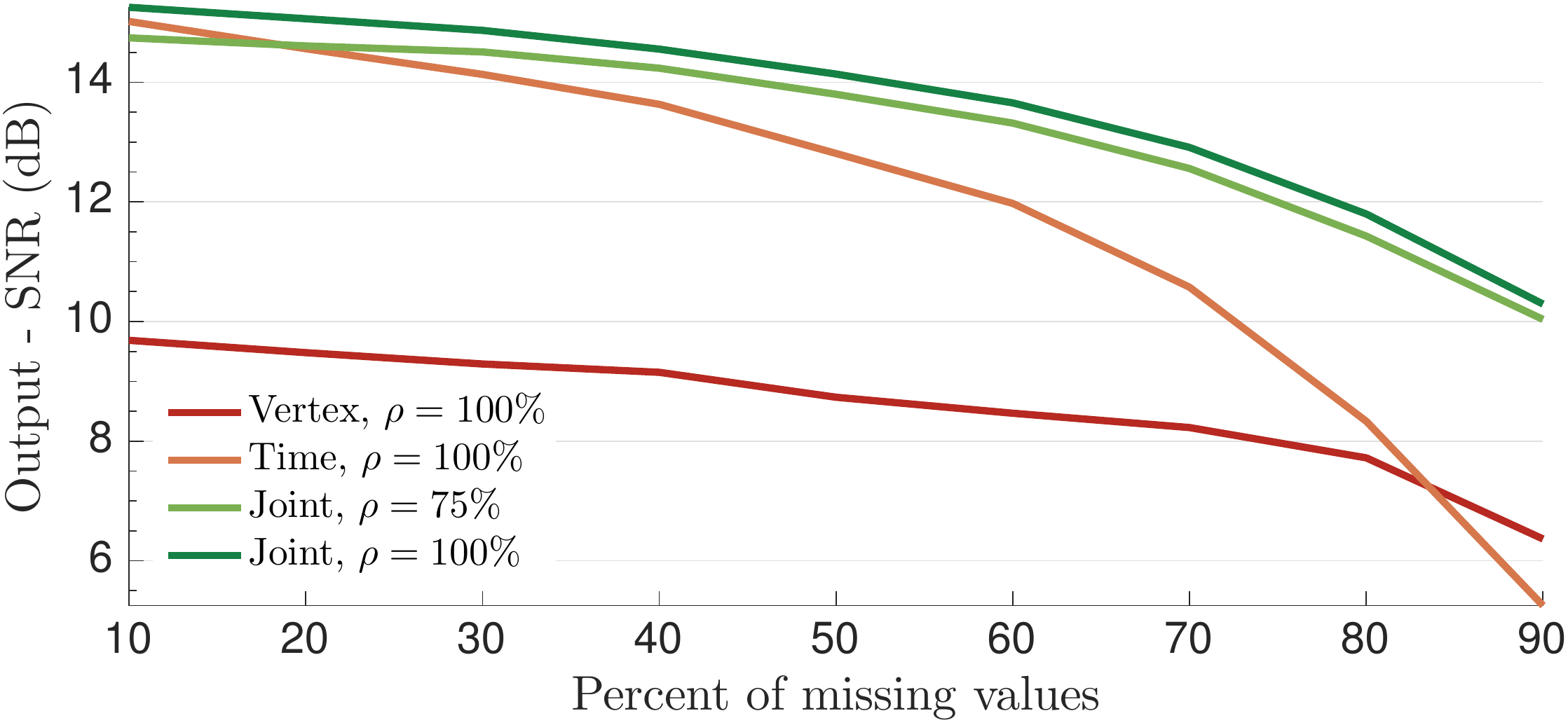}
\caption{Recovery}%
\label{subfigb}%
\end{subfigure}\hfill%
\caption{Experiments on Molene temperatures. The joint approach becomes especially meaningful when the available data are very noisy or are few. The recovery performance is slightly improved when a larger percentage $\rho$ of data are available for training.}
\label{fig:result_molene}
\end{figure}

\paragraph{Denoising.} For this experiment, we add Gaussian random noise to the data and remove the noise thanks to Wiener filter ($\b{A}=\b{I}$ in problem~\eqref{prob:Wiener-opt}). The result is displayed in Figure~\ref{fig:result_molene}. Joint stationarity outperforms time or vertex stationarity especially when the noise level is high. Indeed, joint stationarity allows the estimator to average over more samples. In order to obtain a good denoising, we need a good JPSD estimation. The effect of the dataset size can be observed through the parameter $\rho$, with larger $\rho$ resulting in higher accuracy. Especially for large input SNR, the joint approach becomes particularly meaningful as it outperforms other approaches, even when a very small portion of the data is used for JPSD estimation (whereas the time and vertex based methods $\rho = 1$, meaning that they use the entire dataset for PSD estimation).  

\paragraph{Recovery.} We also consider a recovery problem, where a given percentage of entries of matrix $\X$ is missing. Figure~\ref{fig:result_molene} depicts the recovery error obtained using problem~\eqref{prob:Wiener-opt-noiseless}. Again, we observe a significant improvement over competing methods. This improvement is achieved because the joint approach leverages the correlation both in the time and in the vertex domain: each random variable in a TWSS or VWSS process is dependent on only $T-1$ or $N-1$ other random variables, respectively (rather than $NT-1$ as in the joint case), implying a higher recovery variance.

\section{Conclusion}

This paper proposed a novel definition of (wide-sense) stationarity appropriate for time-varying graph signals. We showed that joint stationarity possess a number of useful properties, that are familiar from the classical setting. Based on our definition, we proposed a Wiener optimization framework and the accompanying PSD estimation method, which together can be used to for solving the problem of inverting a rank-deficient linear system under a jointly stationary input and disturbance. The proposed optimization framework is optimal in the mean-squared error sense and scales well with the number of time samples. In our experiment with a weather dataset, the joint approach was shown to yield a significant benefit over disjoint statistical methods for signal denoising and recovery.       

The longer version of this paper will expand our analysis and evaluate our approach in a larger set of experiments. We will additionally make a detailed complexity analysis and propose solutions to avoid the computationaly expensive diagonalization of the graph Laplacian $\LG$. We remark that our simulations were done using the GSPBOX~\cite{perraudin2014gspbox}, the UNLocBoX~\cite{perraudin2014unlocbox}, and the LTFAT~\cite{ltfatnote030}. The code reproducing all figures will be made available soon.

\bibliographystyle{IEEEtran}
\bibliography{biblio}

\end{document}